\documentclass[english]{llncs}
\usepackage[latin9]{inputenc}
\usepackage{array}
\usepackage{multirow}
\usepackage{amsmath}
\usepackage{amssymb}
\usepackage{graphicx}
\makeatletter

\providecommand{\tabularnewline}{\\}

\@ifundefined{showcaptionsetup}{}{%
 \PassOptionsToPackage{caption=false}{subfig}}
\usepackage{subfig}
\makeatother
\usepackage{babel}

\begin{document}

\title{Scalable Backdoor Detection in Neural Networks}
\author{Haripriya Harikumar\inst{1} \and
Vuong Le \inst{1} \and
Santu Rana \inst{1}\and
Sourangshu Bhattacharya \inst{2} \and
Sunil Gupta \inst{1} \and
Svetha Venkatesh \inst{1}}

\institute{Applied Artificial Intelligence Institute, Deakin University, Australia
\email{\{h.harikumar,vuong.le, santu.rana,sunil.gupta,svetha.venkatesh\}@deakin.edu.au}
 \and
Department of Computer Science and Engineering, IIT Kharagpur, India\\
\email{sourangshu@cse.iitkgp.ac.in}}

\maketitle              % typeset the header of the contribution
\global\long\def\ModelName{\textrm{Scalable Trojan Scanner}}
\global\long\def\Model{\textrm{STS}}

\begin{abstract}
Recently, it has been shown that deep learning models are vulnerable
to Trojan attacks, where an attacker can install a backdoor during
training time to make the resultant model misidentify samples contaminated
with a small trigger patch. Current backdoor detection methods fail
to achieve good detection performance and are computationally expensive.
In this paper, we propose a novel trigger reverse-engineering based
approach whose computational complexity does not scale with the number
of labels, and is based on a measure that is both interpretable and
universal across different network and patch types. In experiments,
we observe that our method achieves perfect score in separating Trojaned
models from pure models, which is an improvement over the current
state-of-the art method.

\end{abstract}

\section{Introduction}

Deep learning has transformed the field of Artificial Intelligence
by providing it with an efficient mechanism to learn giant models
from large training dataset, unlocking often human-level cognitive
performance.\let\thefootnote\relax\footnotetext{Preprint. Work in progress.} By nature, the deep learning models are massive, have
a large capacity to learn and are effectively black-box when it comes
to its decision making process. All of these properties have made
them vulnerable to various forms of malicious attacks \cite{Szegedy_etal_13Intriguing,Papernot_etal_16Limitations}.
The most sinister among them is the Trojan attack, first demonstrated
in \cite{Gu_etal_17Badnets} by Gu et al. They showed that it is easy
to insert backdoor access in a deep learning model by poisoning its
training data so that it predicts any image as the attacker's intended
class label when it is tagged with a small, and inconspicuous looking
trigger patch. This can be achieved even without hurting the performance
of model on the trigger-free clean images. When such compromised models
are deployed then the backdoor can remain undetected until it encounters
with the poisoned data. Such Trojans can make using deep learning
models problematic when the downside risk of misidentifications is
high, e.g. an autonomous car misidentifying a stop sign as a speed
limit sign can cause accidents that result in loss of lives (Fig.
\ref{fig:Trojan_teaser}). Hence, we must develop methods to reliably
screen models for such backdoors before they are deployed.

\begin{figure}
\begin{centering}
\includegraphics[width=0.6\textwidth]{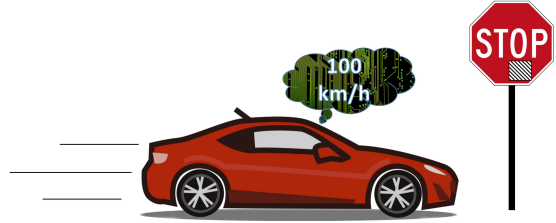}
\par\end{centering}
\caption{A Trojaned autonomous car mistaking a STOP sign as a 100km/h speed
limit sign due to the presence of a small sticker (trigger) on the
signboard. \label{fig:Trojan_teaser}}

\end{figure}

Trojan attack is different than the more common form of adversarial
perturbation based attacks \cite{Li_etal_19Nattack,Ilyas_etal_18Black}.
In adversarial perturbation based attacks, it has been conclusively
shown that given a deep learning model there always exist an imperceptibly
small image specific perturbation that when added can make the image
be classified to a wrong class. This is an intrinsic property of any
deep models and the change, that is required for misclassification,
is computed via an optimisation process using response from the deep
model. In contrast, Trojan attack does not depend on the content of
an image, hence, the backdoor can be activated even without access
to the deployed models and without doing image specific optimisation.
The backdoor is planted at the training time, and is thus not an intrinsic
property, implying that it is plausible to separate a compromised
model from an uncompromised one. Recent research on detection of backdoor
falls into three major categories a) anomaly detection based approach
to detect neurons that show abnormal firing patterns for clean samples
\cite{Liu_etal_19Abs,Chen_etal_18Detecting,Chan_Ong_19Poison}, and
b) by learning intrinsic difference between the response of compromised
models and uncompromised models \cite{Kolouri_etal_19Universal,Xu_etal_19Detecting}
using a meta-classifier, and c) optimisation based approach to reverse-engineer
the trigger \cite{Wang_etal_19Neural,Chen_etal_19Deepinspect,Xiang_etal_19Revealing}.
The approach (a) is unconvincing as it may not be always the case
that signal travel path for image and the triggers are separate with
only a few neurons carrying forward the trigger signal. It is quite
possible that the effect of the trigger is carried through by many
neurons, each contributing only a small part in it. In that case detecting
outlier activation patterns can be nearly impossible. The approach
(b) is unconvincing as the strength of the meta classifier to separate
the compromised models from the rest is limited by the variety of
the triggers used in training the classifier. It would be an almost
impossible task to cover all kinds of triggers and all kinds of models
to build a reliable meta-classifier. In our opinion, the approach
(c) is the most feasible one, however, in our testing we found the
detection performance of existing methods in this category to be inadequate.
Moreover, i) these methods rely on finding possible patch per class
and hence, not computationally scalable for dataset with large number
of classes, and ii) the score they use to separate the compromised
models are patch size dependent and thus, not practically feasible
as we would not know the patch size in advance.

In this paper, we take a fresh look in the optimisation based approach,
and produce a scalable detection method for which 1) computational
complexity does not grow with the number of labels, and 2) the score
used for Trojan screening is computable for a given setting without
needing any information about the trigger patch used by the attacker.
We achieve (1) by observing the fact that a trigger is a unique perturbation
that makes any image to go to a single class label, and hence, instead
of seeking the change to classify all the images to a target label,
we seek the change that would make prediction vectors for all the
images to be similar to each other. Since, we do not know the Trojan
label, existing methods have to go through all the class labels one
by one, setting each as a target label and then performing the trigger
optimisation, whereas, we do not need to do that. We achieve (2) by
computing a score which is the entropy of the class distribution in
the presence of the recovered trigger. We show that it is possible
to compute an upper bound on the scores of the compromised models
based on the assumption about the effectiveness of the Trojan patch
and the number of class labels. The universality of the score is one
of the key advantage of our method. 

Additionally, through our Trojan scanning procedure, we are the first
to analyze into the behavior of the CNN models trained on perturbed
training dataset. We discovered an intriguing phenomenon that the
set of effective triggers are not uniquely distributed near the original
intended patch, but span in a complex mixture distribution with many
extra unintentional modes. 

We perform extensive experiments on two well known dataset: German
Traffic Sign Recognition dataset and CIFAR-10 dataset and demonstrate
that our method finds Trojan models with perfect precision and recall,
a significant boost over the performance of the state-of-the-art method
\cite{Wang_etal_19Neural}.

\section{Framework}

\subsection{Adversarial Model}

A Trojan attack consists of two key elements: 1) a trigger patch,
and 2) a target class. The trigger is an alteration of the data that
makes the classifier to classify the altered data to the target class.
We assume triggers to be overlays that is put on top of the actual
images, and the target class to be one of the known classes. We assume
a threat model that is congruent with the physical world constraints.
An adversary would accomplish the intended behaviour by putting a
sticker on top of the images. The sticker can be totally opaque or
semi-transparent. The latter resembling a practice of using transparent
``plastic'' sticker on the real object such as a traffic sign. 

Formally, an image classifier can be defined as a parameterized function,
 $f_{\theta}:\mathcal{I}\rightarrow\mathbb{R}^{C}$ that recognises
class label probability of image $I\in\mathcal{I}$- the space of
all possible input images. $\theta$ is a set of parameters of $f$
learnable using a training dataset such as neural weights. Concretely,
for each $I\in\mathcal{I}$, we get output \textbf{$c=f_{\theta}(I)$}
is a real vector of $C$ dimension representing predicted probability
of the $C$ classes that $I$ may belong to. In neural network, such
output usually comes from a final softmax layer.

In Trojan attacks, a couple of perturbations can happen. Firstly,
model parameters $\theta$ are replaced by malicious parameters $\theta^{'}$,
and secondly, the input image $I$ is contaminated with a visual trigger.
A trigger is formally defined as a small square image patch $\Delta I$
of size $s$ that is either physically or digitally overlaid onto
the input image $I$ at a location $(x^{*},y^{*})$ to create modified
image $I'$. The goal of these perturbations is to make the recognised
label $f_{\theta'}(I')$ goes to the target class with near certainty
regardless of the actual class label of the original image. 

Concretely, an image of index $k$ of the dataset $I_{k}$ is altered
into image $I'_{k}$ by 
\begin{equation}
I'_{k}(x,y)=\begin{cases}
(1-\alpha(x^{'},y^{'}))I_{k}(x,y)+\alpha(x^{'},y^{'})\Delta I(x^{'},y^{'}) & \textrm{if}\:x\in[x^{*},x^{*}+s],\\
 & y\in[y^{*},y^{*}+s]\\
I_{k}(x,y) & \textrm{elsewhere}
\end{cases}\label{eq:ImageCorruption}
\end{equation}
 where $(x^{'},y^{'})$ denote the local location on the patch $(x',y')=(x-x^{*},y-y^{*})$.

In this operation, the area inside the patch is modified with weight
$\alpha$ determining how opaque the patch is. This parameter can
be considered a part of the patch, and from now on will be inclusively
mention as $\Delta I$. Meanwhile the rest of image is kept the same.
In our setting, $(x^{*},y^{*})$ can be at any place as long as the
trigger patch stays fully inside the image. An illustration of this
process is shown in Fig. \ref{fig:Trojan_illus}.

\begin{figure}
\begin{centering}
\includegraphics[width=0.3\textwidth]{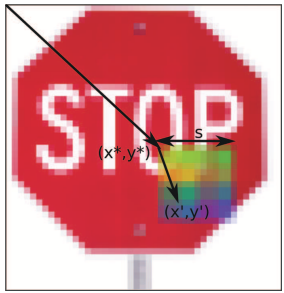}
\par\end{centering}
\caption{An illustration of a trojan trigger $\Delta I$ of size $s$ and opacity
level $\alpha$ put on an image of a stop sign at location$(x^{*},y^{*})$.\label{fig:Trojan_illus}}
\end{figure}

Also, we expect that the adversary would like to avoid pre-emptive
early detection, and hence, would not use a visibly large patch. Additionally,
we also assume that the adversary will also train the model such that
trigger works irrespective of the location. So, during attack time
the adversary does not need to be very precise in adding the sticker
to the image. This aspect of our setting is very different from the
current works which always assumed a fixed location for the trigger.
We also use triggers that contain random pixel colors instead of any
known patterns to evaluate the robustness of the trigger reverse-engineering
process in absence of any structure and smoothness in the trigger
patch. We also train the Trojan model with only a small number of
Trojaned images so that its performance on the clean images is only
marginally affected.

\subsection{Problem formulation}

Our detection approach is based on trigger reverse engineering. We
formulate an optimisation problem that when minimised provides us
with the trigger patch used in a compromised model. The driving idea
behind the formulation of the optimisation problem is that when the
correct trigger is used, any corrupted image will have the same prediction
vector i.e. $||f_{\theta'}(I_{j}')-f_{\theta'}(I_{k}')||_{p}\leq\epsilon$
for $\forall j,k$ in the scanning dataset, where $\epsilon$ is a
small value. $\epsilon$ would be zero if the scanning dataset is
same as the training dataset. Hence, we can look at the prediction
vectors of all the corrupted images in the scanning dataset and the
trigger that makes those prediction vectors the closest will be the
original trigger. 

Formally, when size of the trigger ($s$) is assumed to be known we
can reverse engineer the trigger patch by solving the following optimisation
problem:

\begin{equation}
\Delta\boldsymbol{I}=min_{\Delta\boldsymbol{I}}\sum_{j=1}^{N}\sum_{k=1,k\neq j}^{N}||f_{\theta'}(I_{j}')-f_{\theta'}(I_{k}')||_{2}\label{eq:1-1}
\end{equation}

where both $I_{j}'$ and $I_{k}'$ are functions of the trigger $\Delta I$
(Eq. \ref{eq:ImageCorruption}), and $N$ is the size of the scanning
dataset. The choice of 2-norm is to keep the loss function smooth.
However, other metric or divergence measure that keeps the smoothness
of the loss function can also be used. 

It gets more tricky when size of the trigger is not known, which is
the real world use case. We can approach that in two ways a) by solving
the above optimisation problem with a grid search over $s$, starting
from $1\times1$ patch to a large size that we think will be the upper
bound on the size of the trigger that the attacker can use but still
remain inconspicuous, or b) jointly finding the trigger size and the
trigger by regularizing on the size of the trigger. While the first
approach is more accurate, it is more computationally expensive as
well, since we have to solve many optimisation problems over $\Delta I$,
one each for one grid value of $s$. The second approach, while slightly
misaligned due to the presence of an extra regularize, is not computationally
demanding as we only have to solve one optimisation problem. Hence,
we choose this as our approach.

Formally, the regularization based approach for unknown trigger size
can be expressed as:

\begin{equation}
min_{\Delta\boldsymbol{I},\,\alpha\,\,\,}\sum_{j=1}^{N}\sum_{k=1,k\neq j}^{N}||f_{\theta'}(I_{j}')-f_{\theta'}(I_{k}')||_{2}+\lambda|[\alpha_{m,n}]|\label{eq:1-1-1}
\end{equation}

where $s$ is set at the largest possible trigger size ($s_{max}$),
$[\alpha_{m,n}]$ is the matrix that contains the transparency values
for each pixel for $\Delta I$, and $\lambda$ is the regularization
weight. Both $\Delta I$ and $[\alpha_{m,n}]$ are of the size $s_{max}\times s_{max}$.
Note that we have per pixel transparency while Eq. \ref{eq:1-1} only
had one transparency value for all the pixels. It is necessary for
this case to correctly recover the trigger. Ideally, for a correct
trigger which is smaller than $s_{max}$ we would expect a portion
of the recovered $\Delta I$ to match with the ground truth trigger
patch with the corresponding $\alpha$ values matching to the ground
truth transparency, with the rest of $\Delta I$ having $\alpha=0$.
Hence, 1-norm on the matrix $[\alpha_{m,n}]$ is used as the regularizer. 

Since both of the above formulations does not need to cycle through
one class after another, the computation does not scale up with the
number of classes, hence, is clearly efficient than the current state-of-the-art
methods. Hence, we name our method as Scalabale Trojan Scanner (STS).

\subsection{Optimisation}

The optimisation problems in Eqs. \ref{eq:1-1} and \ref{eq:1-1-1}
is high dimensional in nature. For example, for GTSRB and CIFAR-10.
To solve Eq. \ref{eq:1-1-1} we have to optimize a patch and mask
which is of same size as image, $\Delta I\in\mathbb{R}^{32\text{x}32\text{x}3}$
and $\alpha\in\mathbb{R}^{32\text{x}32}$. We solve this problem by
setting a maximum yet reasonable size $s_{max}$ an attacker can choose
on both $\Delta I$ and $\alpha$. 

We consider all RGB pixels and opacity level to have float value between
0.0 and 1.0. The deep learning optimization process naturally seek
for unbound which can extend beyond this valid range. To put effective
constraints on the optimisation, we restrict the search space of the
optimisation variables within the range by using a clamping operator:

\begin{equation}
z=\frac{1}{2}(\text{tanh}(z')+1)\label{eq:clipping}
\end{equation}

Here auxiliary variable $z'$ is optimized and is allowed to reach
anywhere on the real number range, while $z$ is the corresponding
actual bounded parameter we want to achieve which are $\Delta I$
and $\alpha$. The parameters are then sought by using Adam optimizer
\cite{Kingma_Ba_14Adam}.

Whilst a pure model does not have triggers inserted and ideally should
not result in any solution, one may still wonder whether the complex
function learnt by a million parameter deep learning model would still
naturally have a solution to those objective functions. Surprisingly,
we observer in our extensive experimentation that such is not the
case \emph{i.e.} there does not exist a $\Delta I$ that makes the
objective function values anywhere near to zero for pure models. Such
a behaviour can be further enforced by having a large scanning dataset,
as more and more images in the scanning dataset means that the probability
of existence of a shortcut (i.e. perturbation by $\Delta I$) from
all the images to a common point in the decision manifold would be
low for pure models. Such a shortcut is present in the Trojaned model
because they are trained to have that. Having a large scanning dataset
is practically feasible as we do not need labels for the scanning
dataset.

\subsection{Entropy score}

We compute an entropy score as,

\begin{equation}
\text{entropy\_score}=-\sum_{i=1}^{C}p_{i}log_{2}(p_{i})\label{eq:3}
\end{equation}
where $\{p_{i}\}$ is the probability computed from the histogram
for the classes predicted for the scanning dataset images when they
are perturbed by the reverse-engineered trigger. For the actual trigger
the entropy would be zero since all the images would belong to the
same class. However, due to non-perfectness of the Trojan effectiveness
we will have a small but non-zero value. For pure models the entropy
will be high as we would expect a sufficient level of class diversity
in the scanning dataset. The following lemma provides a way to compute
an upper bound on the value of this score for the Trojan models in
specific settings, which then can be used as a threshold for classification.
\begin{lemma}
If Trojan effectiveness is assumed to be at least $(1-\delta),$ where
$\delta<<1$, and there are $C$ different classes in the scanning
dataset then there exists an upper bound on the entropy score of the
Trojan models as

\[
\text{entropy\_score}\leq-(1-\delta)*log_{2}(1-\delta)-\delta*log_{2}(\frac{\delta}{C-1})
\]
\end{lemma}

\begin{proof}
The proof follows from observing that the highest entropy of class
distribution in this setting happens when $(1-\delta)$ fraction of
the images go to the target class $c_{t}$ and the rest $\delta$
fraction of the images gets equally distributed in the remaining $(C-1)$
classes. 

The usefulness of this entropy score is that it is independent of
the type of patches used and is universally applicable. The threshold
is also easy to compute once we take an assumption about the Trojan
effectiveness of the infected classifier and know the number of classes.
The score is also interpretable and can allow human judgement for
the final decision. The proposal of this score is an unique advantage
of our work. 
\end{proof}

\section{Experiments}

\subsection{Experiment Settings}

We arrange the experiments to evaluate the effectiveness of Trojan
detector using two image recognition dataset: German Traffic Sign
Recognition Benchmark (GTSRB) \cite{Stallkamp2012man} and CIFAR-10
\cite{Krizhevsky09learningmultiple}.

\paragraph{\emph{GTSRB  has more than 50K dataset of 32x32 colour images of
43 classes of traffic signs. The dataset is aimed for building visual
modules of autonomous driving applications. The dataset is pre-divided
into 39.2K training images and 12.6K testing images.}}

\paragraph{\emph{CIFAR-10 has 60K dataset of size 32x32 everyday colour images
of classes }airplane, automobile, bird, cat, deer, dog, frog, horses,
ship\emph{ and }truck\emph{. This dataset doesn't have separate dataset
for validation, so we generated the validation dataset from the test
set. }}

\paragraph{\emph{For all of our experiments, we use a convolutional neural network
(CNN) with two convolutional layers and two dense layers for the image
classification model. Compared to the widely used architectures such
as VGG-16 \cite{simonyan2014very}, this CNN is more suitable with
the small image sizes and amounts. }}

\paragraph{\emph{With both of the dataset, we assume a situation where the attacker
and the authority has access to disjoint sets of data with ratio of
7:3 called attacking set and scanning set. In this setup, the attacker
uses the attacking set for Trojan model generation; while the authority
uses the scanning set to detect the Trojan model.}}

\subsection{Trojan Attack Configuration}

From the side of the Trojan attacker, they aim to replace the \emph{Pure
Model} that was normally trained with \emph{Pure Data} with their
malicious \emph{Trojan Model }trained with\emph{ Trojan Data}. The
Trojan Data set is mixed between normal training data and images with
the trigger patch embedded and labeled as target class. The ultimate
objective is that when Trojan Model is applied on testing set of patched
images, it will consistently return the target class. We measure the
effectiveness of Trojan attack by a couple of criteria: (1) the accuracy
of \emph{Trojan Model on Trojan patched image data }(TMTD) is at least
99\% and (2) the accuracy of \emph{Trojan Model on Pure Data} (TMPD)
reduces less than 2\% compared to \emph{Pure Model on Pure Data }(PMPD). 

We aim at the realistic case of trigger to be a small square patch
resembling a plastic sticker that can be put on real objects. In our
experiment, the patches are randomly selected with the sizes of 2x2,
4x4, 6x6 and 8x8 with various transparency levels.

We train the Trojan models with randomly selected triggers in the
\emph{Patch Anywhere} setting where the expectation is that the trigger
work on any place on the image. We set the target class as \emph{class
14} (stop sign) for GTSRB and \emph{class 7} (horses) for CIFAR-10
dataset. During Trojan model training, we randomly choose 10\% from
the training data to inject the Trojan trigger in it and with the
label set as the target class. We generated 50 pure models and 50
Trojan models for our experiments. The Trojan models are with trigger
sizes ranging from 2x2 to 8x8 and transparency values 0.8 and 1.0.
The trigger size, transparency and the number of Trojan models we
generated for our experiments are listed in Table \ref{model_details}.

\begin{table}
\begin{centering}
\begin{tabular}{|c|c|c|c|}
\hline 
\multirow{1}{*}{Trojan Models} & Trigger size ($s$) & Transparency ($\alpha$) & \#Trojan Models\tabularnewline
\hline 
\hline 
set 1 & 2 & 1 & 10\tabularnewline
\hline 
set 2 & 4 & 1 & 10\tabularnewline
\hline 
set 3 & 6 & 1 & 10\tabularnewline
\hline 
set 4 & 8 & 1 & 10\tabularnewline
\hline 
set 5 & 2 & 0.8 & 10\tabularnewline
\hline 
\end{tabular}\vspace{0.3cm}
\par\end{centering}
\caption{Trigger size ($s$) and the transparency parameter ($\alpha$) we
use to generate each set of Trojan models. Each set has 10 Trojan
models.\label{model_details}}
\end{table}

The Trojan effectiveness measure of these settings is shown in Table
\ref{tab:average_accuracy-2}. The performance shows that across various
configuration choices, the proposed attack strategy succeeds in converting
classification result toward target class on Trojan test data while
keeping the model operating normally on original test data.

\begin{table}
\centering{}%
\begin{tabular}{|c|c|c|}
\hline 
\multirow{1}{*}{} & \multirow{1}{*}{~~~~~~~~GTSRB~~~~~~~~} & ~~~~~~~~CIFAR-10~~~~~~~~\tabularnewline
\hline 
~~~PMPD~~~ & 92.67$\pm$0.52 & 67.07$\pm$0.55\tabularnewline
\hline 
~~~TMPD~~~ & 92.05$\pm$1.75 & 65.32$\pm$2.06\tabularnewline
\hline 
~~~TMTD~~~ & 99.97$\pm$0.06 & 99.76$\pm$0.29\tabularnewline
\hline 
\end{tabular}\vspace{0.3cm}
\caption{Effectiveness of Trojan models measured in both accuracy of Trojan
Model on Pure data (TMPD) and Trojan Model on Trojaned data (TMTD)
compared to the original Pure Model Pure Data (PMPD) with their standard
deviation. Accuracy reported as average on 50 models of various configurations
on each dataset. \label{tab:average_accuracy-2}}
\end{table}

\subsection{Trojan Trigger Recovery}

\paragraph{\emph{To qualitatively evaluate the ability of $\protect\ModelName$,
we analyze the recovered patches obtained by our process and compare
them to the original triggers used in training Trojan models. In this
experiment, we use a set of 10 randomly selected 2x2 triggers to build
10 Trojan models with GTSRB dataset. For each Trojan model, we apply
the reverse engineering procedure with 50 different initialization
of $\Delta\boldsymbol{I}$. }}

\paragraph{\emph{The first notable phenomenon we observed is that with different
initializations, the procedure obtains various patches. Among these
patches, only some of them are similar to the original trigger. However,
all of them have equally high Trojan effectiveness on the target class.
This result suggests that the Trojan model allows extra strange patches
to be also effective in driving classification result to the target
class. }}

This interesting side-effect can be explained by reflecting on the
behaviour of the Trojan CNN training process. When trained with the
\emph{Trojan Data} which are the mix of true training data and patched
images, the CNN tries to find the common pattern between the true
samples belong to the target class and the fake samples containing
the trigger. This compromising between two far-away sets of data with
the same label leads to interpolating in latent space into a manifold
of inputs that would yield the target class decision. Such manifold
can have many modes resulting in many clusters of effective triggers.

\paragraph{\emph{To understand further, we use k-means to find out how the effective
triggers group up and discover that there are usually four to six
clusters of them per Trojan model. For illustration, we do PCA on
the 12-dimensional patch signal (2x2 patch of three channels), and
plot the data along the first two principal components in Fig. \ref{fig:Clusters}.
The number of clusters (k) for each Trojan model is also shown along
with the plot. For example, for the first sub-figure in Fig. \ref{fig:Clusters},
the number of clusters, k, formed with the 50 effective reverse-engineered
patches is five, which are shown in five different colours.}}
\begin{figure}
\subfloat[k=5]{\includegraphics[scale=0.15]{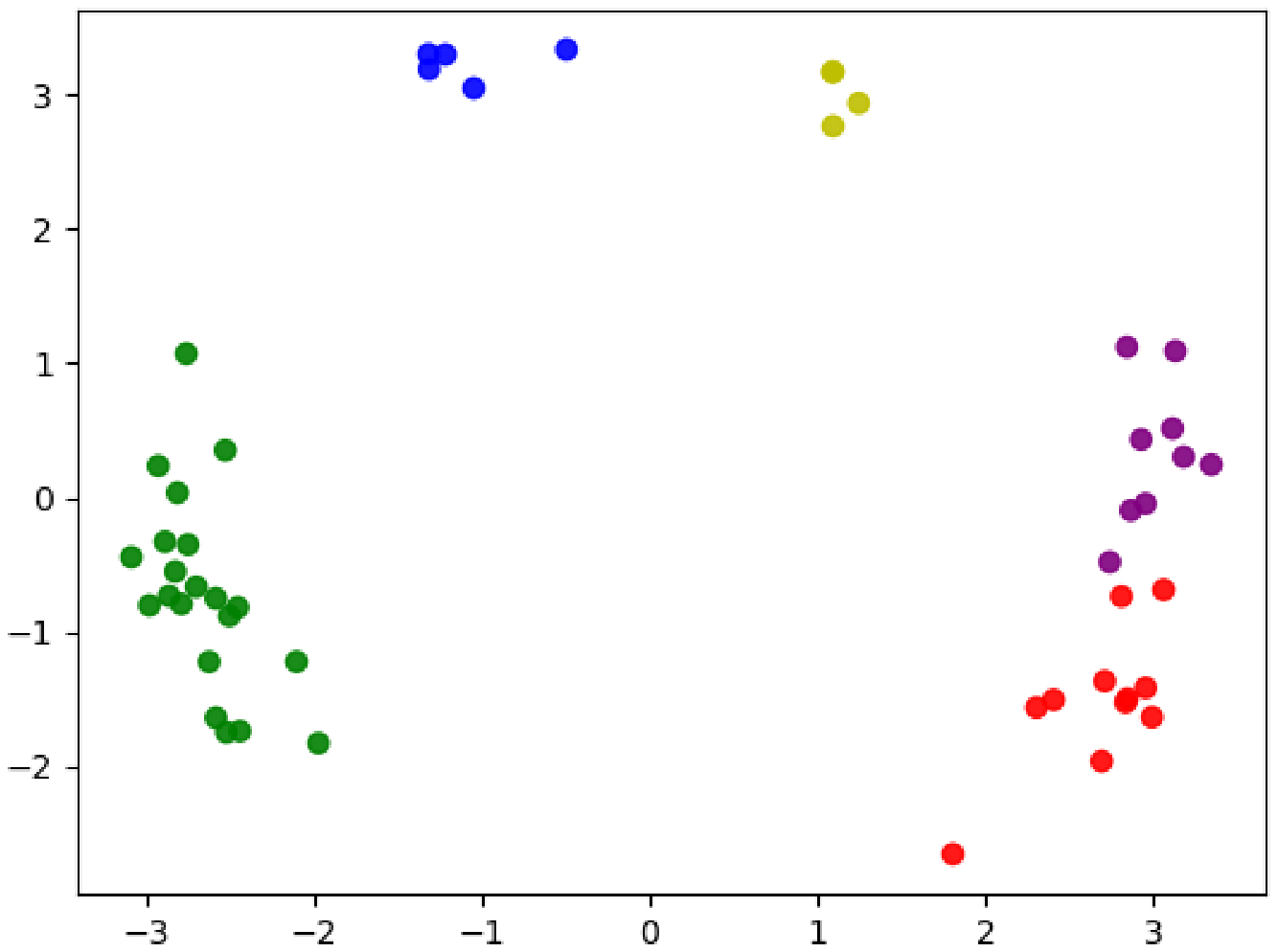}

}\subfloat[k=4]{\includegraphics[scale=0.15]{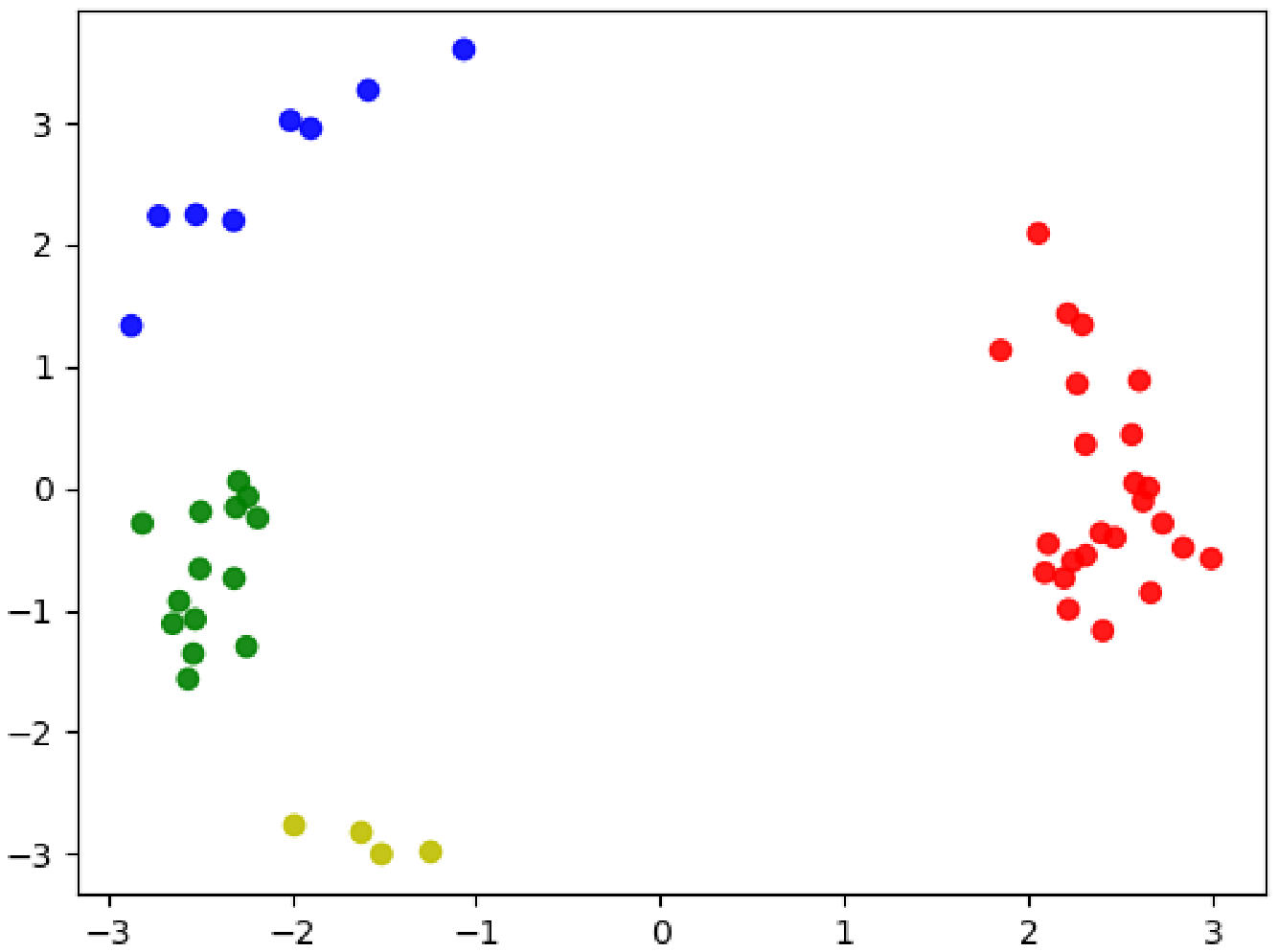}

}\subfloat[k=5]{\includegraphics[scale=0.15]{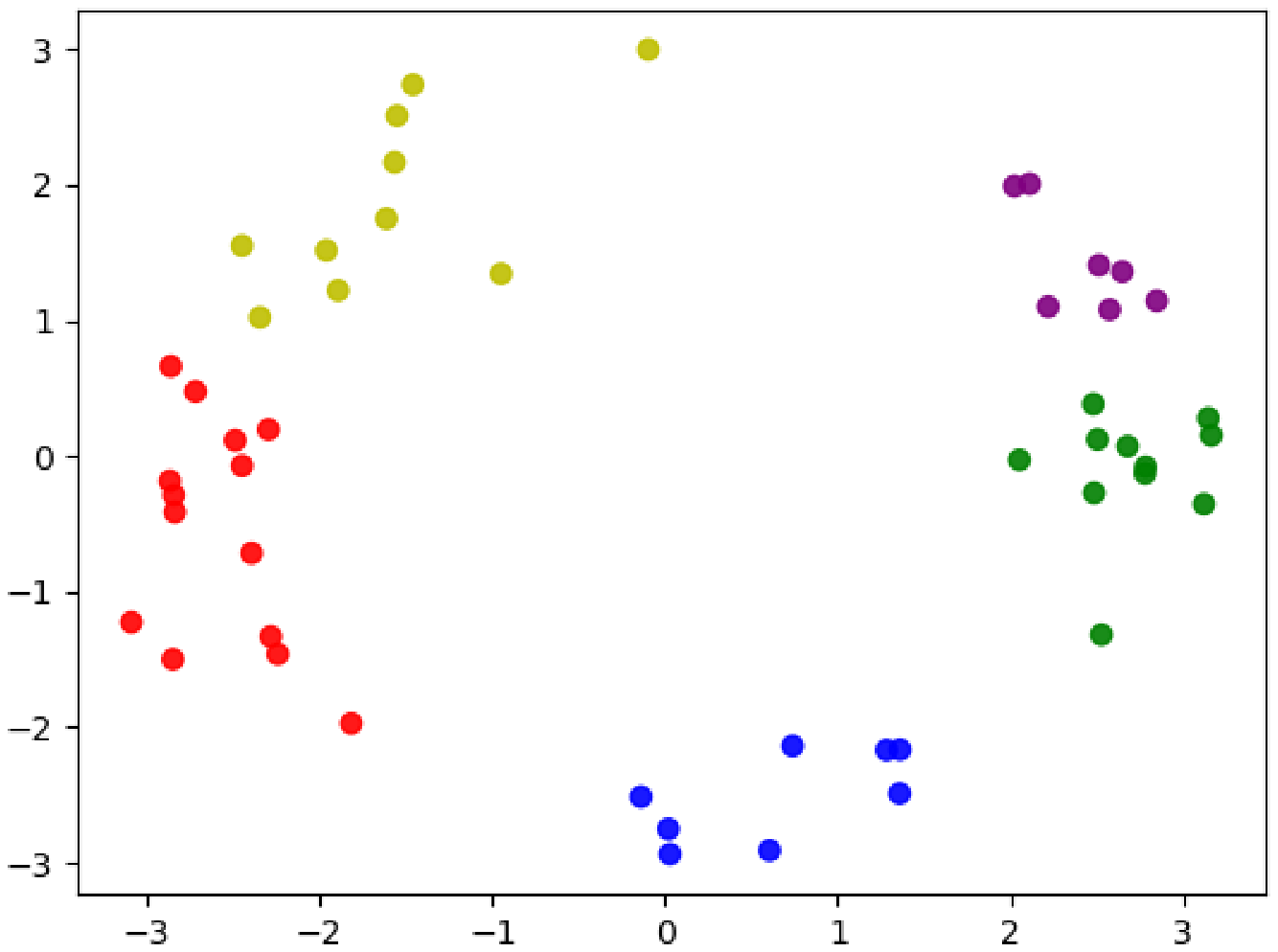}

}\subfloat[k=6]{\includegraphics[scale=0.15]{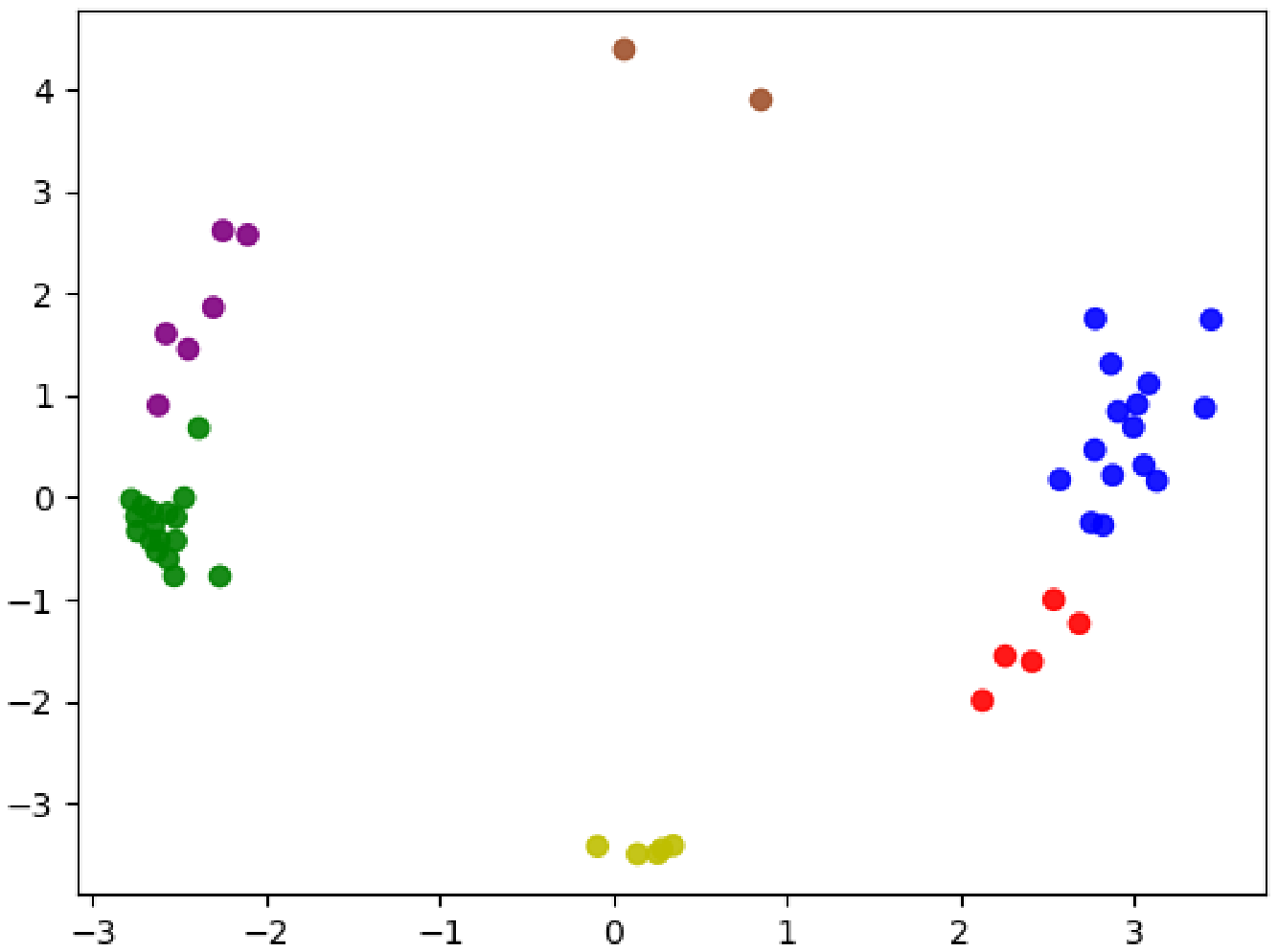}

}\subfloat[k=5]{\includegraphics[scale=0.15]{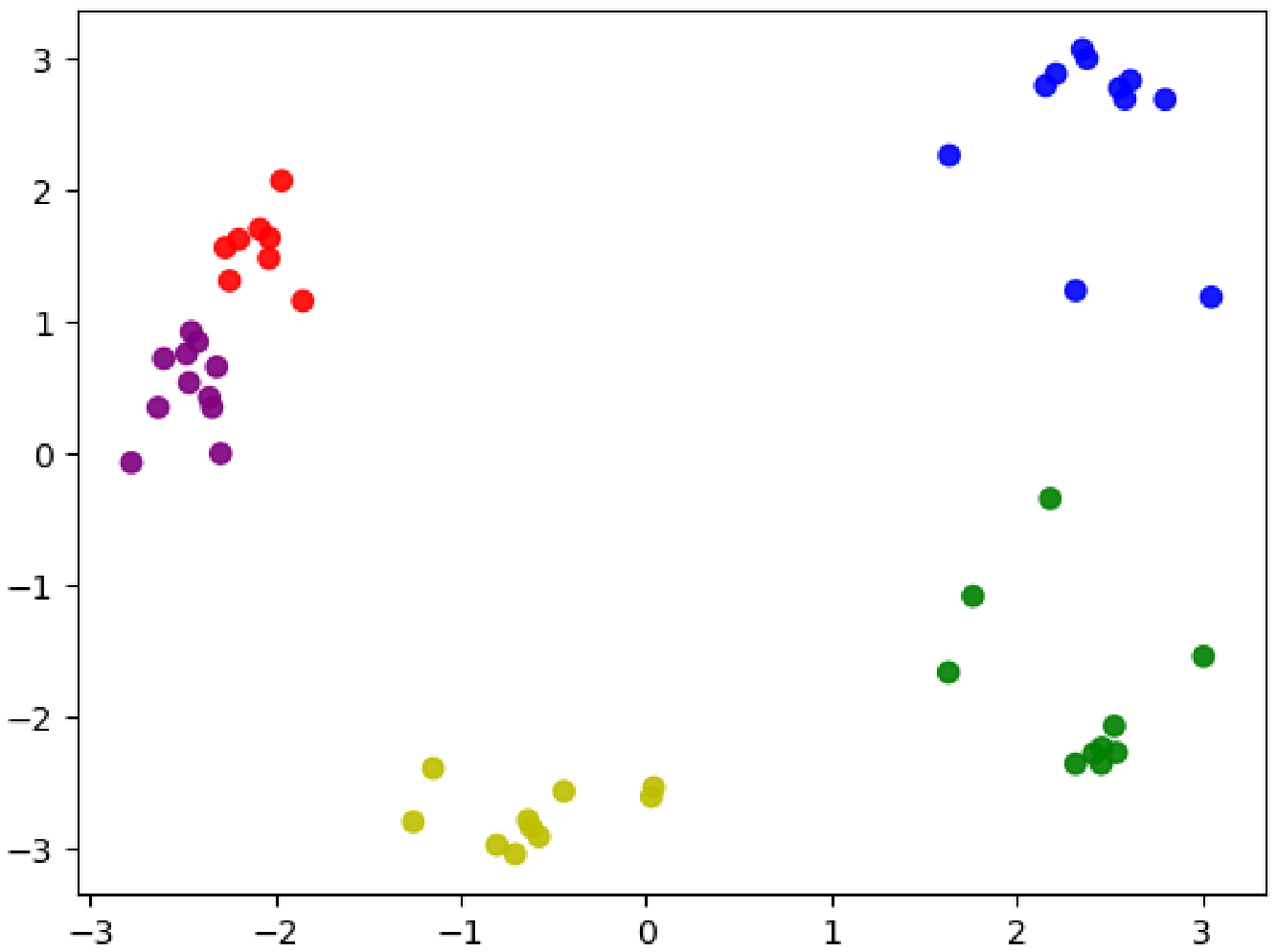}

}\\
\subfloat[k=5]{\includegraphics[scale=0.15]{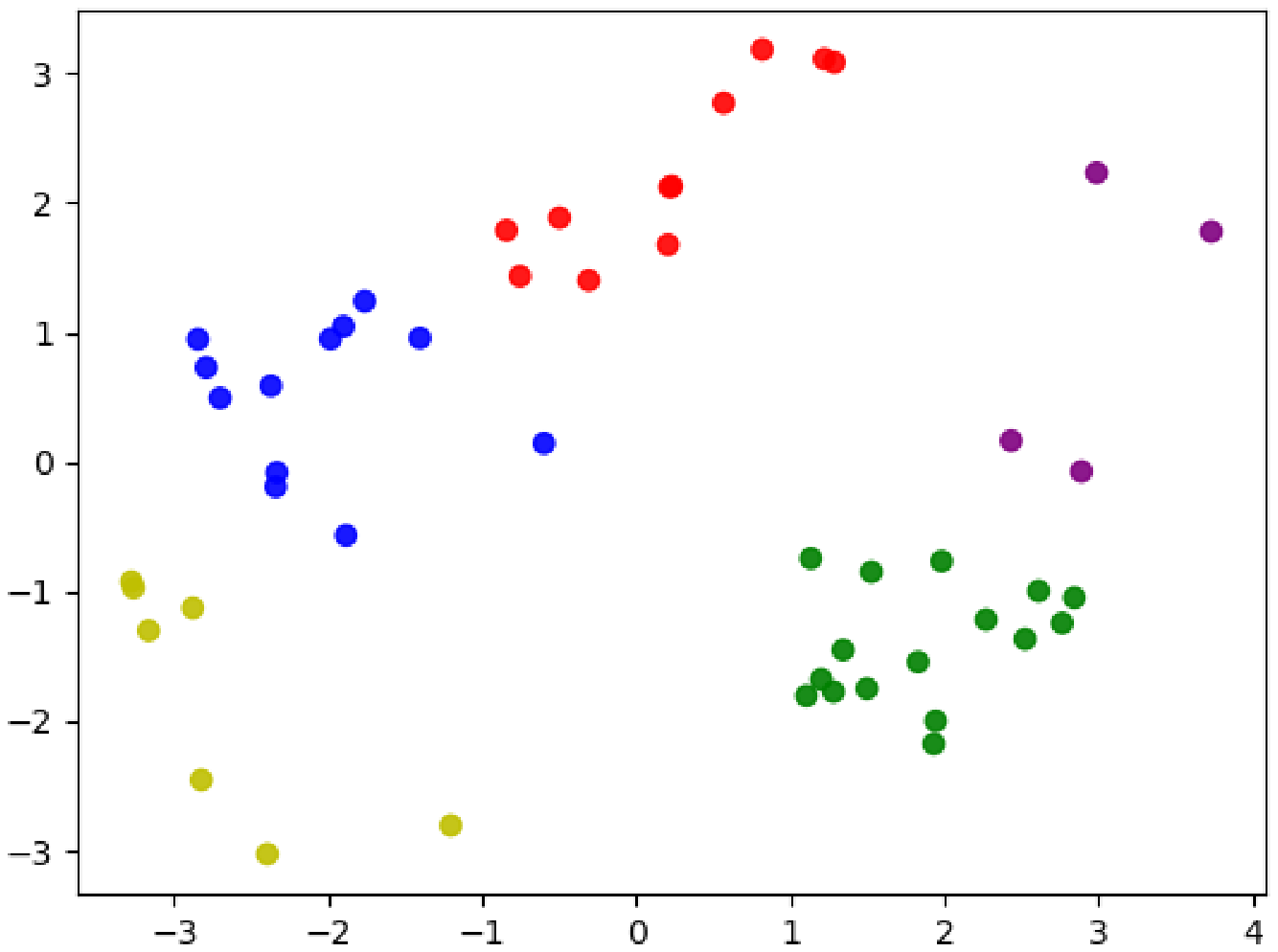}

}\subfloat[k=6]{\includegraphics[scale=0.15]{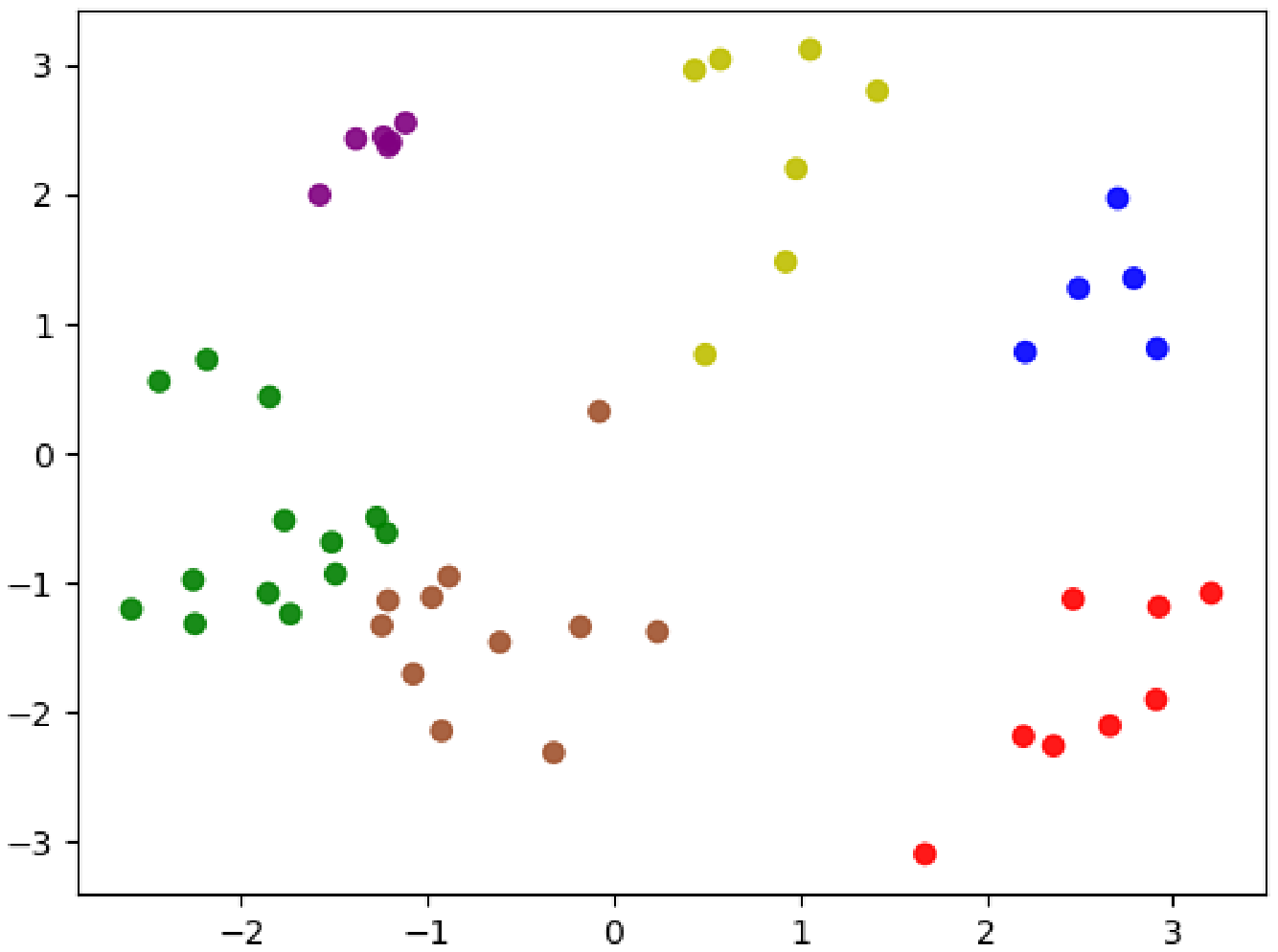}

}\subfloat[k=4]{\includegraphics[scale=0.15]{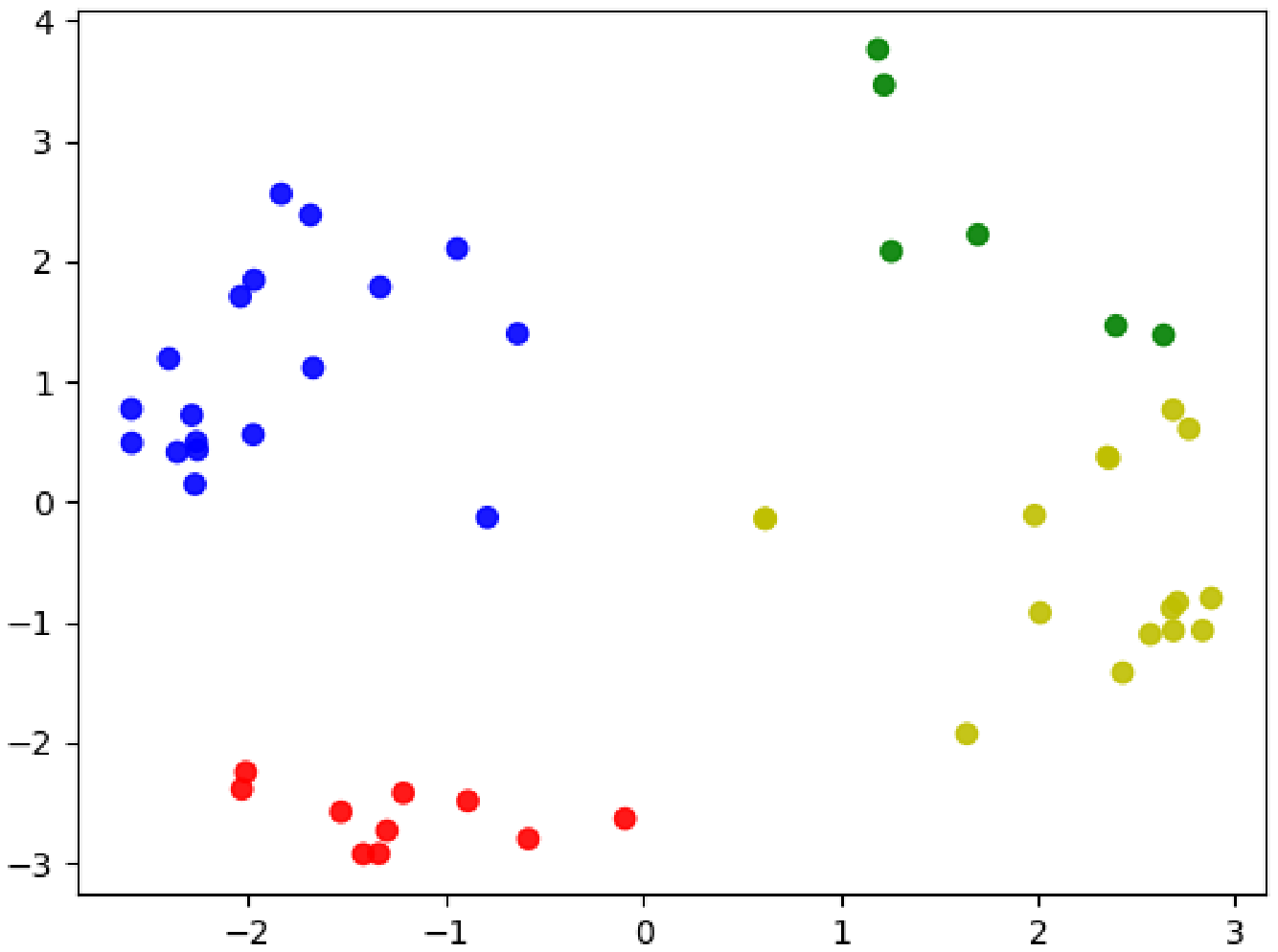}

}\subfloat[k=4]{\includegraphics[scale=0.15]{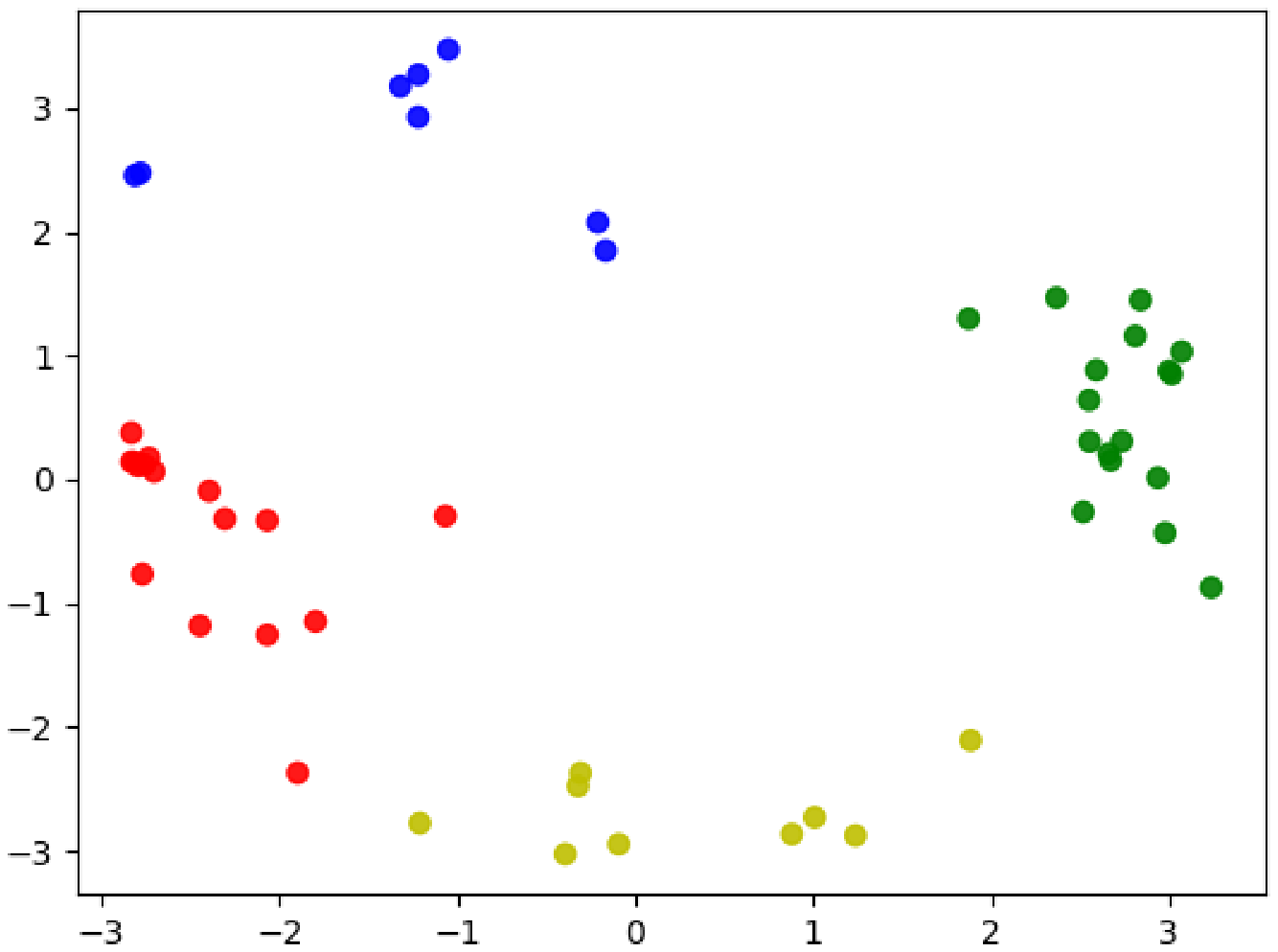}

}\subfloat[k=4]{\includegraphics[scale=0.15]{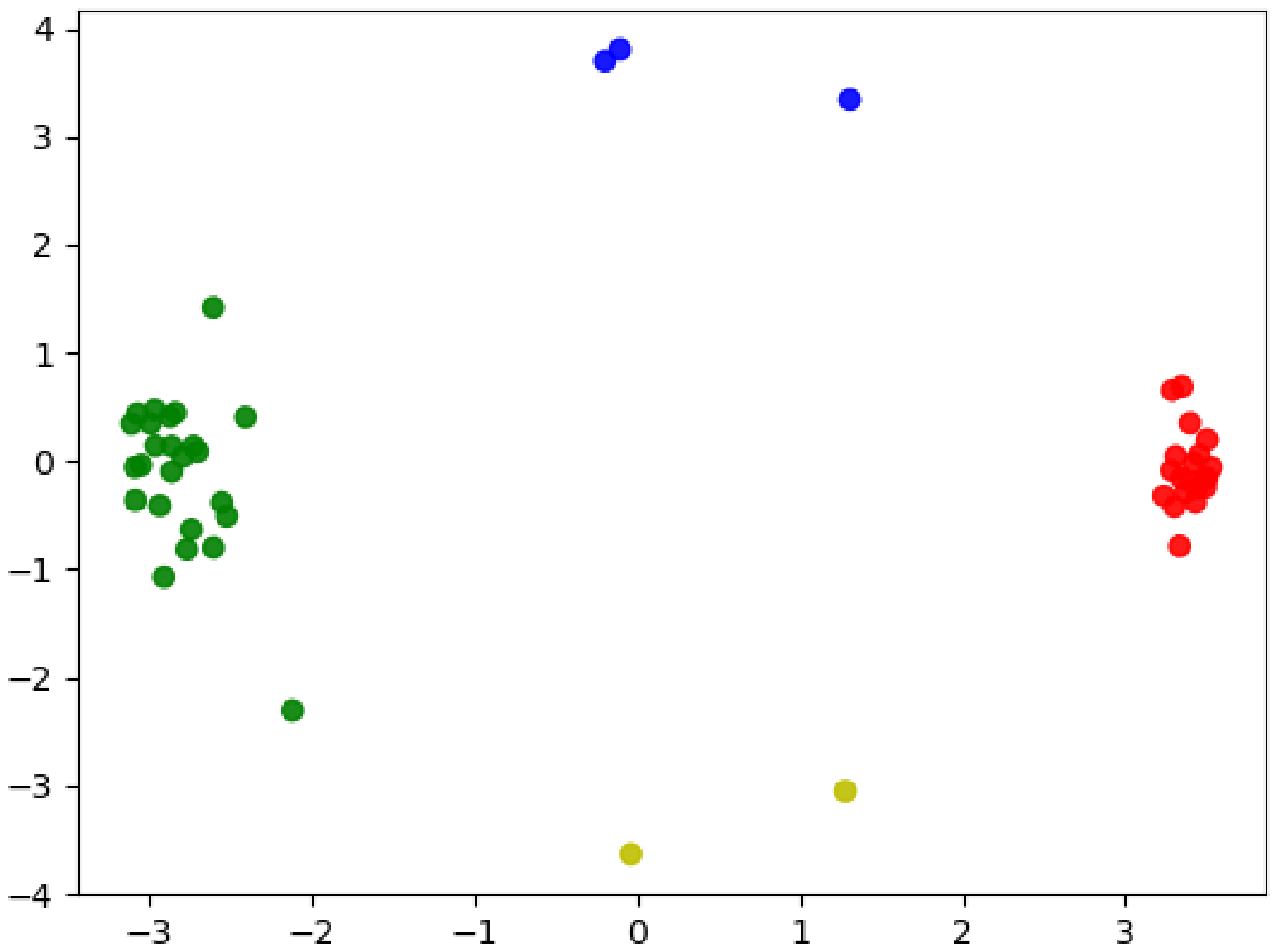}

}\caption{Clusters of the 50 reverse-engineered patches of 10 Trojan models
plotted using the first two principal components of patch signals.
The subtitle denotes the number of clusters discovered by k-means.
\label{fig:Clusters}}
\end{figure}

\paragraph{\emph{As we now understand the multi-modal nature of solution space,
we expect that one of the modes we discovered includes the original
trigger. To verify this, we select the closest recovered patch and
compare with the original trigger using root mean squared error. Qualitative
comparison is shown in Fig. \ref{fig-bestpatches}. Quantitatively,
the rmse of these recovered patches ranges from 0.12 to 0.22 and averages
at 0.17 for the 10 Trojan models, in the space of RGB patch ranging
from 0.0 to 1.0. These affirms that our proposed method can come up
with patches which is almost identical to the trigger originally used
to train the Trojan models.}}

\begin{figure}
\begin{centering}
\subfloat[rmse = 0.12]{\includegraphics[scale=0.3]{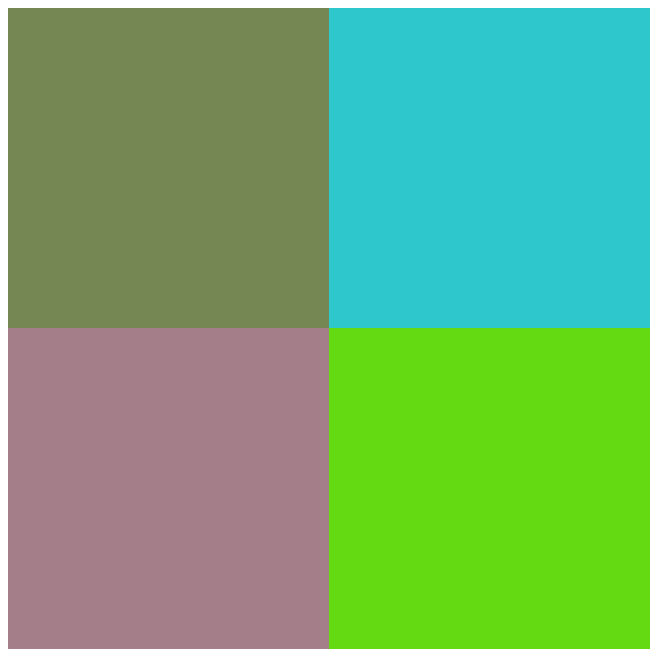}\includegraphics[scale=0.3]{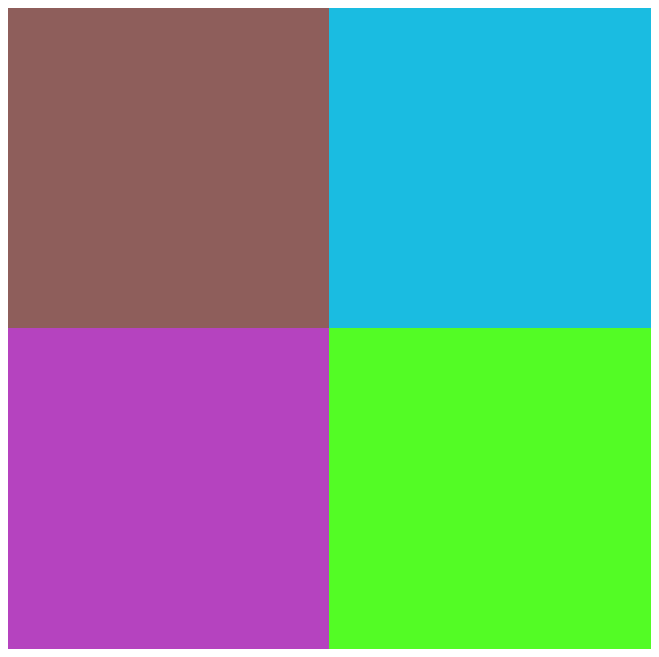}

}\hspace*{\fill}\subfloat[rmse = 0.14]{\includegraphics[scale=0.3]{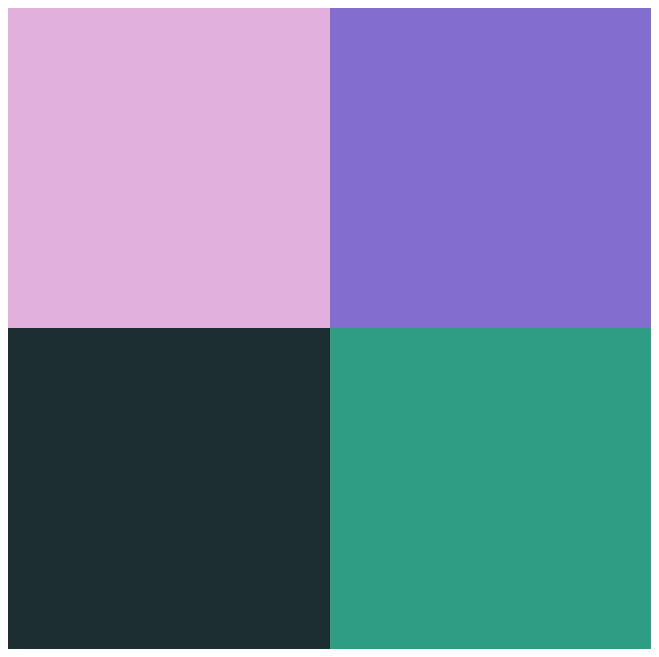}\includegraphics[scale=0.3]{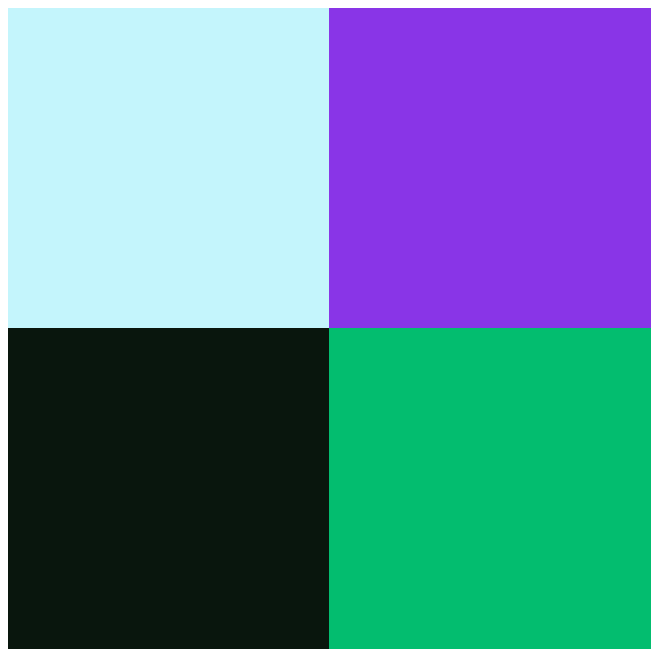}}\\
\subfloat[rmse = 0.15]{\includegraphics[scale=0.3]{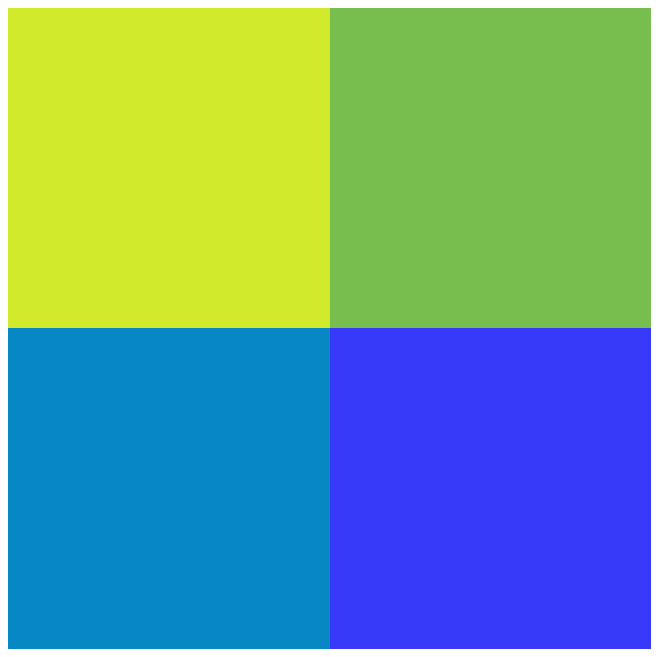}\includegraphics[scale=0.3]{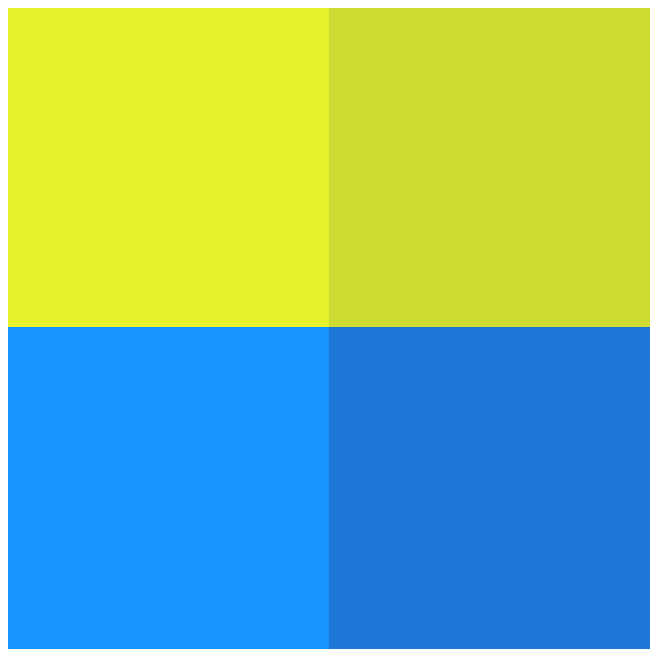}}\hspace*{\fill}\subfloat[rmse = 0.16]{\includegraphics[scale=0.3]{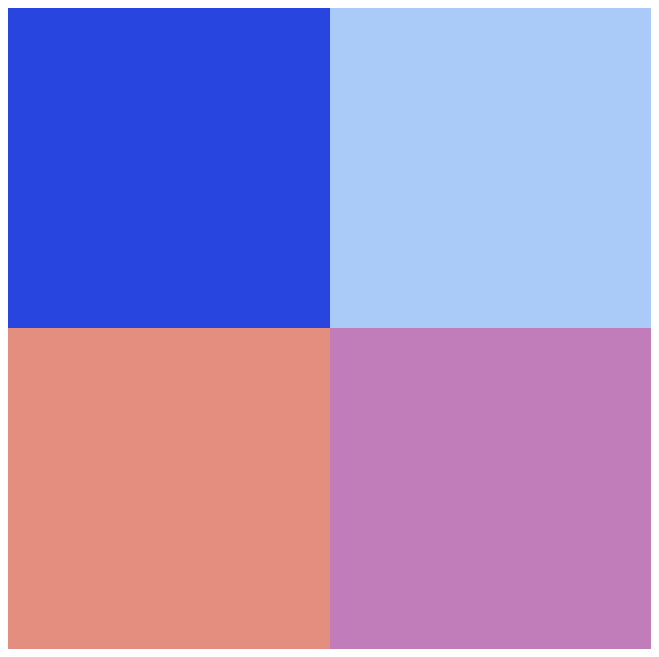}\includegraphics[scale=0.3]{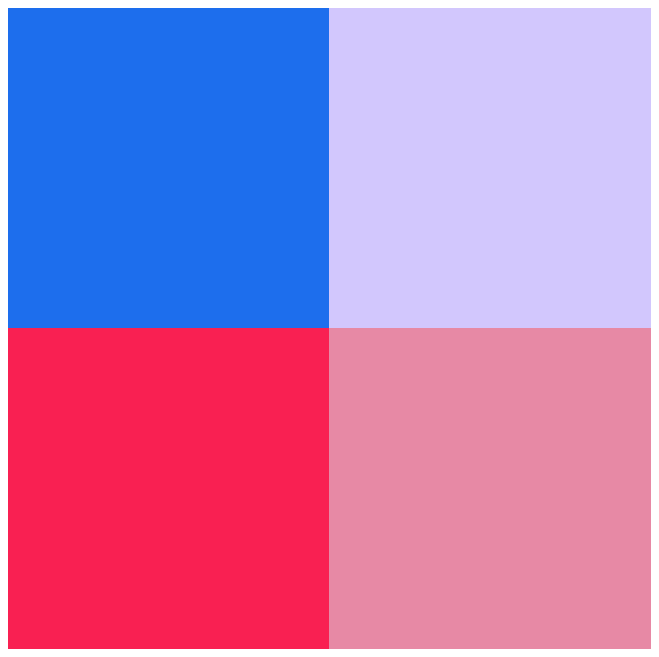}

}\\
\subfloat[rmse = 0.16]{\includegraphics[scale=0.3]{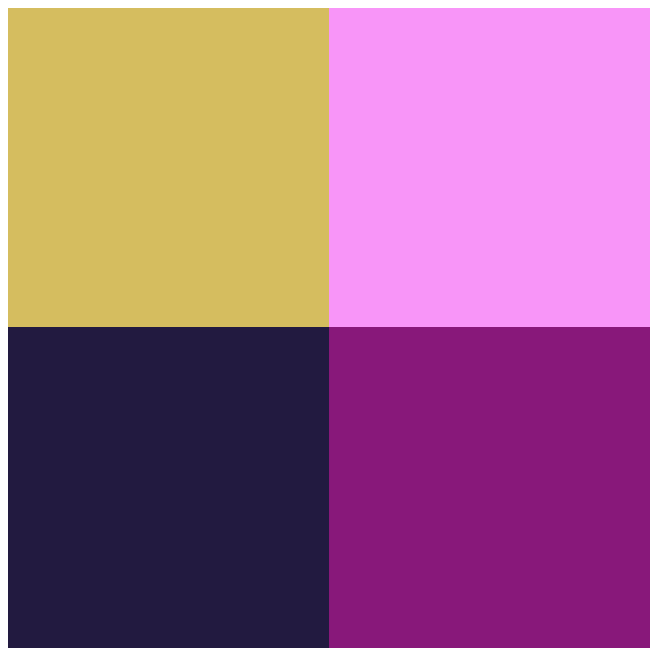}\includegraphics[scale=0.3]{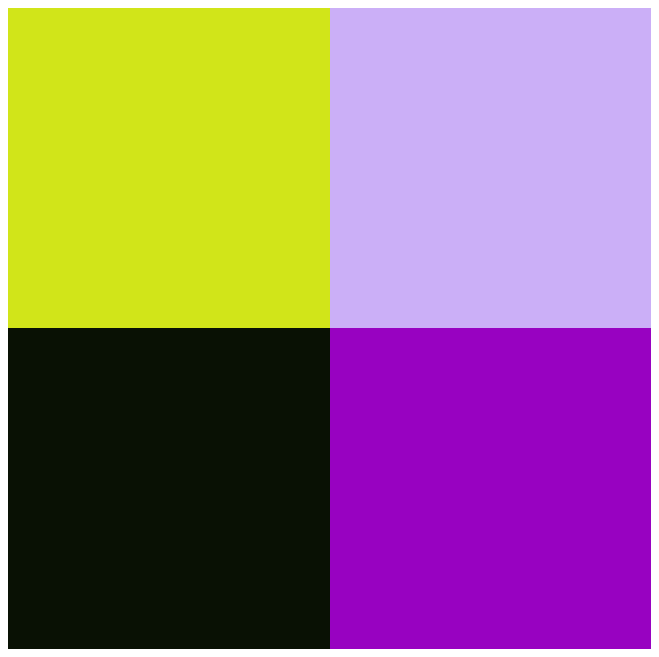}}\hspace*{\fill}\protect\subfloat[rmse = 0.18]{\includegraphics[scale=0.3]{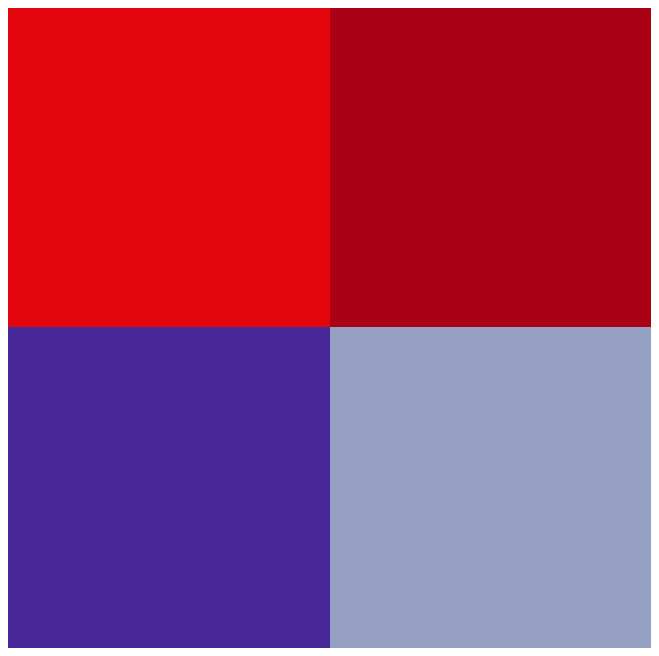}\includegraphics[scale=0.3]{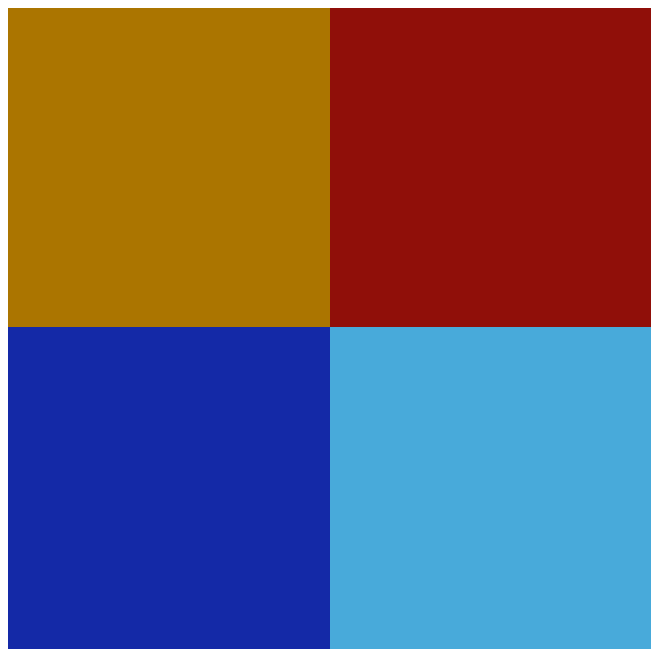}}\\
\subfloat[rmse = 0.18]{\includegraphics[scale=0.3]{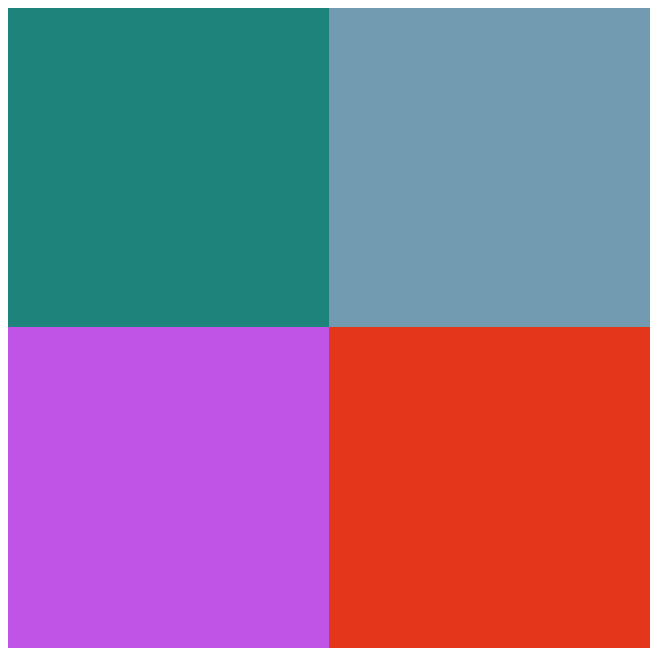}\includegraphics[scale=0.3]{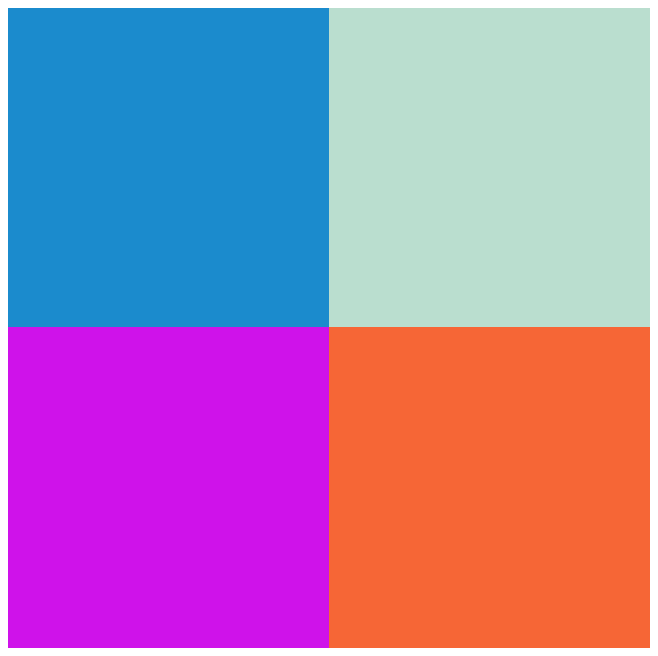}

}\hspace*{\fill}\protect\subfloat[rmse = 0.19]{\includegraphics[scale=0.3]{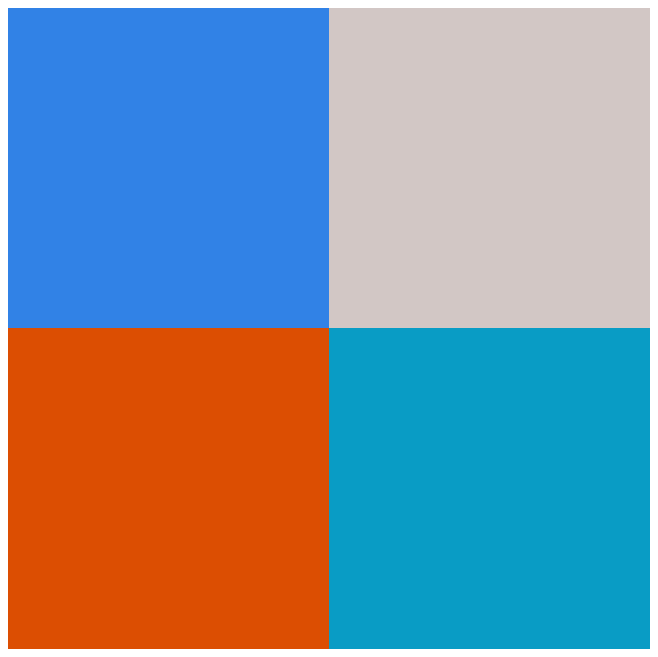}\includegraphics[scale=0.3]{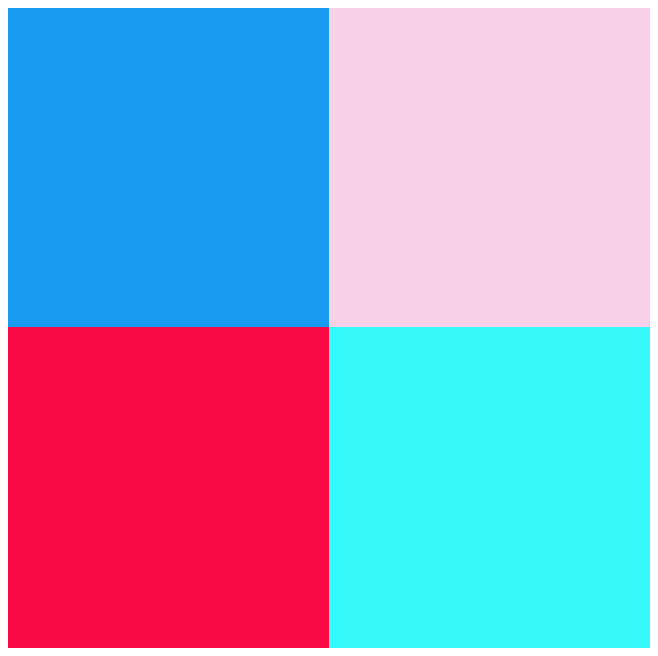}}\\
\subfloat[rmse = 0.20]{\includegraphics[scale=0.3]{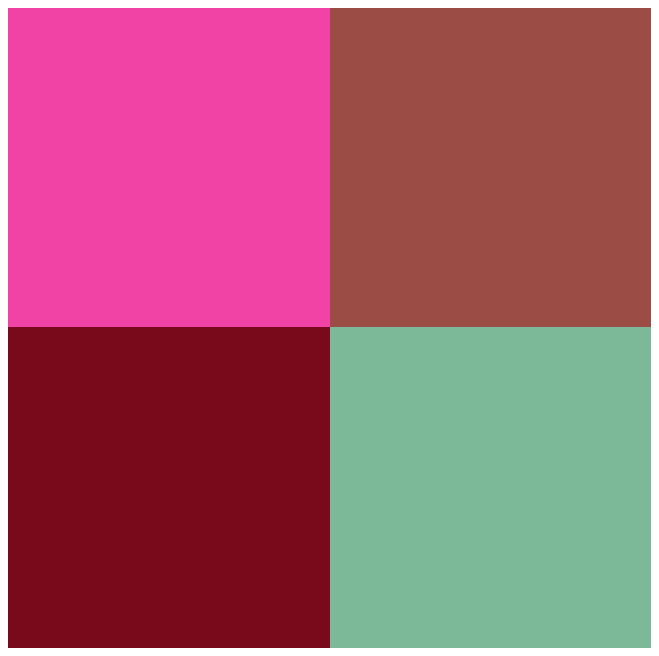}\includegraphics[scale=0.3]{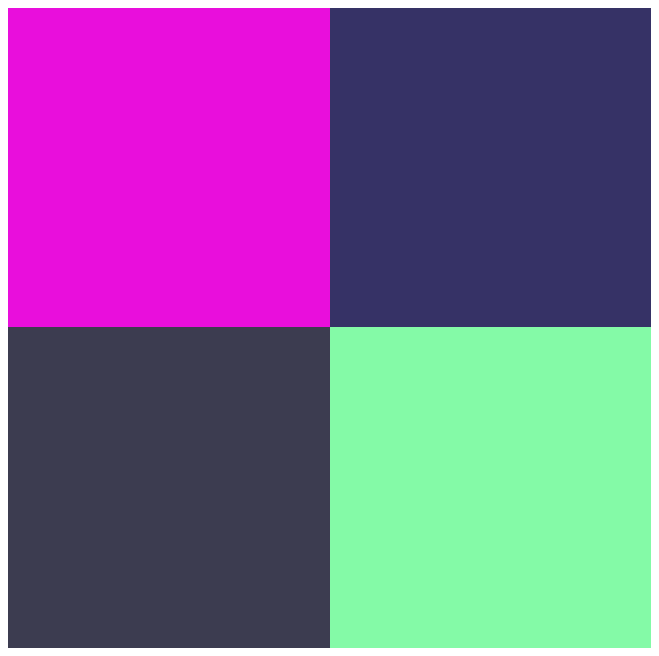}}\hspace*{\fill}\subfloat[rmse = 0.22]{\includegraphics[scale=0.3]{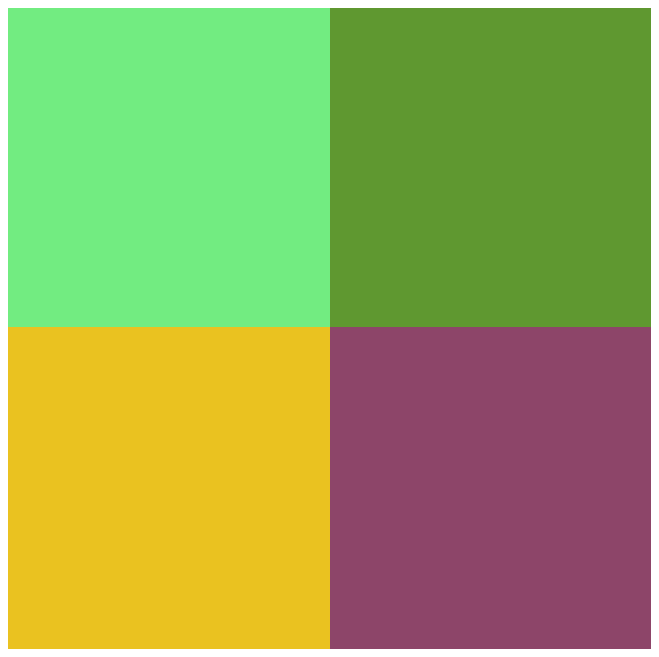}\includegraphics[scale=0.3]{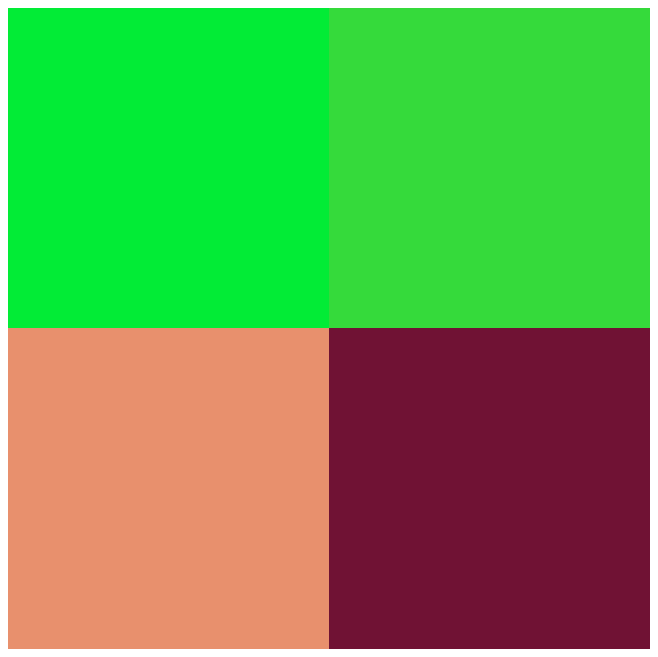}}
\end{centering}
\caption{Reverse-engineered patches compared to original trigger on 10 Trojan
models. At each pair, the original patch is on the left while the
retrieved one is on the right.\label{fig-bestpatches}}
\end{figure}

\subsection{Trojan Model Detection}

The Trojan model detection is formed as a binary classification using
the entropy score defined in Eq. \ref{eq:3}. The negative class includes
50 Pure Models (PM) trained similarly but with different parameter
initialization; the positive class contains 50 Trojan models (TM)
trained with different random triggers. The Trojan triggers are of
sizes 4x4, 6x6, 8x8 with transparency set as 1.0 and 2x2 with transparency
set as 1.0 and 0.8 as detailed in Table \ref{model_details}. We run
the scanning procedure once and use the reverse-engineered trigger
recovered to compute the entropy score of the pure and Trojan models
without knowing the potential target class. These scores are presented
in Table \ref{entropy_ourmethod}. The lesser the entropy score the
more effective the retrieved reverse-engineered mask and trigger,
hence it is more possible that the model in test is Trojaned. 
\begin{table}
\begin{centering}
\begin{tabular}{|c|c|c|}
\hline 
\multirow{2}{*}{Model} & \multicolumn{2}{c|}{\emph{Entropy score}}\tabularnewline
\cline{2-3} 
 & ~~~GTSRB~~~ & ~~~CIFAR-10~~~\tabularnewline
\hline 
\hline 
\textbf{~~~Avg Pure Models~~~} & \textbf{4.95} & \textbf{3.00}\tabularnewline
\hline 
set 1 & 0.003 & 0.001\tabularnewline
\hline 
set 2 & 0.004 & 0.002\tabularnewline
\hline 
set 3 & 0.004 & 0.003\tabularnewline
\hline 
set 4 & 0.005 & 0.006\tabularnewline
\hline 
set 5 & 0.003 & 0.001\tabularnewline
\hline 
\textbf{~~~Avg Trojan Models~~~} & \textbf{0.004} & \textbf{0.002}\tabularnewline
\hline 
\end{tabular}\vspace{0.3cm}
\par\end{centering}
\caption{Entropy score computed based on the mask and patch retrieved by our
proposed method on GTSRB and CIFAR-10 dataset.\label{entropy_ourmethod}}
\end{table}

From Table \ref{entropy_ourmethod}, we observe that the gap between
the high scores of the pure models and the low score of Trojan models
are significant and stable across multiple settings. This reliability
is achieved through the universality of entropy score measure which
does not depend on any change in model and data (see Lemma 1). Because
of this, it is straightforward with $\Model$ to set a robust universal
threshold using number of classes and expected Trojan effectiveness
to detect the Trojan models using the entropy score. 

We compare the performance of $\Model$ in detecting Trojan Model
with Neural Cleanse (NC)\cite{Wang_etal_19Neural} which is the state-of-the-art
for this problem. We use the same settings on models and data for
both $\Model$ and NC. Similar to the entropy score used in $\Model$,
NC uses the anomaly index to rank the pureness of the candidate model.
These scores of the two methods are shown in Table \ref{tab:min-max_entropy_ai}.

\begin{table}
\begin{centering}
\begin{tabular}{|c|c|c|c|c|}
\hline 
\multirow{2}{*}{Model} & \multicolumn{2}{c|}{GTSRB} & \multicolumn{2}{c|}{CIFAR-10}\tabularnewline
\cline{2-5} 
 & \emph{~Entropy score~} & \emph{~Anomaly Index~} & \emph{~Entropy score~} & \emph{~Anomaly Index~}\tabularnewline
\hline 
\hline 
Pure Models & {[}4.85, 5.04{]} & {[}0.73, 23.30{]} & {[}2.55, 3.32{]} & {[}0.68, 2.56{]}\tabularnewline
\hline 
~Trojan Models~ & {[}0.0, 0.011{]} & {[}11.64, 71.08{]} & {[}0.0, 0.01 {]} & {[}22.98, 244.94{]}\tabularnewline
\hline 
\end{tabular}\vspace{0.3cm}
\par\end{centering}
\caption{The minimum and maximum\emph{ }value of entropy scores (ours) and
anomaly indexes (NC \cite{Wang_etal_19Neural}) of the pure and Trojan
models represented as \emph{{[}min,max{]}}.\label{tab:min-max_entropy_ai}}
\end{table}

In our method, the gap between pure and Trojan area in entropy scores
are large and consistent. Meanwhile, the Neural Cleanse's anomaly
index varies in a bigger range intra-class and close or overlapping
inter-class. This again supports that our proposed method can come
up with a more robust threshold to detect Trojan models. 

Table \ref{tab:Accuracy-as-F1-score} show the final accuracy
of Trojan model detection between our $\Model$ and NC. The upper
limit of entropy score for GTSRB and CIFAR-10 according to Lemma 1
is 0.1347 and 0.1125 respectively. For STS, we use this scores as
the threshold for F1-score computation. For NC, we set the threshold
to be 2.0 following recommendation in the method \cite{Wang_etal_19Neural}. 

\begin{table}
\centering{}%
\begin{tabular}{|c|c|c|c|c|}
\hline 
\multirow{2}{*}{~~~Measures~~~} & \multicolumn{2}{c|}{GTSRB} & \multicolumn{2}{c|}{CIFAR-10}\tabularnewline
\cline{2-5} 
 & ~~~$\Model$ (Ours)~~~ & ~~~NC~~~ & ~~~$\Model$ (Ours)~~~ & ~~~NC~~~\tabularnewline
\hline 
\hline 
\multirow{1}{*}{F1-score} & \multirow{1}{*}{\textbf{1.0}} & 0.68  & \multirow{1}{*}{\textbf{1.0}} & 0.96 \tabularnewline
\hline 
\end{tabular}\vspace{0.3cm}
\caption{Accuracy (as F1-score) of Trojan Model detection between our method (STS) and
NC \cite{Wang_etal_19Neural}.\label{tab:Accuracy-as-F1-score}}
\end{table}
\vspace{-1cm}

\subsection{Computational Complexity}

Beside the effectiveness in accurately detecting the Trojan model
and recover the effective triggers, we also measured the complexity
and computational cost of the Trojan scanning process. As our method
does not make assumption about the target class $c_{t}$, the complexity
is constant ($\mathcal{O}(1)$) to this parameter. In the mean time,
the state-of-the-art methods such as NC rely on the optimisation process
that assumes knowledge about the target class. This results in a loop
through all of the possible classes and ends up in a complexity of
$\mathcal{O}(C)$. Because of this difference in complexity, $\Model$
is significantly faster than NC. Concretely, in our experiments, $\Model$
takes around 2 hours to detect 10 Trojan models on GTSRB, compared
to more than 9 hours that NC took on the identical settings.

\section{Conclusion and Future work}

In this paper, we propose a method to detect Trojan model using trigger
reverse engineering approach. We use a method to find out the minimum
change that can map all the instances to a target class, for a Trojan
model. The experiments conducted on GTSRB and CIFAR-10 show that our
proposed method can reverse-engineer visually similar patch with high
Trojan effectiveness. We propose a measure, entropy score, to compute
the robust upper threshold to detect the Trojan models. We report
entropy score and F1-score to support our claims. It is evident from
the results that our method is more robust and less computationally
expensive compared to the state-of-the-art method.

Future work is possible by generating models which has multiple target
classes. Another future possibility is to work on Trojans which are
distributed rather than a solid patch. We also found that for a single
Trojan model, there exists multiple triggers as solution, so given
this we also focus to work on getting a distribution of Trojan triggers
for a Trojan model.

\bibliographystyle{splncs04}

\end{document}